\documentclass{ecai}

\usepackage{latexsym}
\usepackage{hyperref}
\usepackage{amssymb}
\usepackage{amsmath}
\usepackage{amsthm}
\usepackage{microtype} 
\usepackage{booktabs}
\usepackage{multirow}
\usepackage{enumitem}
\usepackage{graphicx}
\usepackage{color}
\usepackage{algorithm}
\usepackage{algpseudocodex}
\usepackage{lipsum}

\usepackage[capitalize,noabbrev]{cleveref}

\crefname{equation}{}{}

\newtheorem{theorem}{Theorem}
\newtheorem{lemma}[theorem]{Lemma}

\newtheorem{proposition}[theorem]{Proposition}

\newcommand{\R}{{\mathbb R}} %

\newcommand{\cQ}{\mathcal{Q}} %
\newcommand{\cS}{\mathcal{S}} %
\newcommand{\cA}{\mathcal{A}} %
\newcommand{\cD}{\mathcal D}

\DeclareMathOperator*{\argmin}{arg\,min}
\DeclareMathOperator*{\argmax}{arg\,max}

\def\Lap{\textrm{Laplace}} %

\newcommand{\lrs}[1]{\left [#1 \right]} %

\NewDocumentCommand{\E}{o}{\mathbb E\IfValueT{#1}{\lrs{#1}}}
\newcommand{\V}{\mathbb{V}}
\NewDocumentCommand{\Var}{o}{\V\IfValueT{#1}{\lrs{#1}}}

\begin{document}

\begin{frontmatter}

\paperid{6112}

\title{Improving Actor-Critic Training with Steerable Action-Value Approximation Errors}

\author[A]{\fnms{Bahareh}~\snm{Tasdighi}\thanks{Corresponding author: \texttt{tasdighi@imada.sdu.dk}}}
\author[A]{\fnms{Nicklas}~\snm{Werge}}
\author[A]{\fnms{Yi-Shan}~\snm{Wu}} 
\author[A]{\fnms{Melih}~\snm{Kandemir}} 

\address[A]{Department of Mathematics and Computer Science \\ University of Southern Denmark}

\begin{abstract}
Off-policy actor-critic algorithms have shown strong potential in deep reinforcement learning for continuous control tasks. Their success primarily comes from leveraging pessimistic state-action value function updates, which reduce function approximation errors and stabilize learning. However, excessive pessimism can limit exploration, preventing the agent from effectively refining its policies. Conversely, optimism can encourage exploration but may lead to high-risk behaviors and unstable learning if not carefully managed. To address this trade-off, we propose \emph{Utility Soft Actor-Critic (USAC)}, a novel framework that allows independent, interpretable control of pessimism and optimism for both the actor and the critic. USAC dynamically adapts its exploration strategy based on the uncertainty of critics using a utility function, enabling a task-specific balance between optimism and pessimism. This approach goes beyond binary choices of pessimism or optimism, making the method both theoretically meaningful and practically feasible. Experiments across a variety of continuous control tasks show that adjusting the degree of pessimism or optimism significantly impacts performance. When configured appropriately, USAC consistently outperforms state-of-the-art algorithms, demonstrating its practical utility and feasibility.
\end{abstract}

\end{frontmatter}

\section{Introduction} \label{sec:introduction}

Deep reinforcement learning (RL) faces significant challenges when navigating the complex landscape of high-dimensional state spaces and non-linear state-action functions, especially in continuous control tasks. The deadly triad of function approximation, off-policy learning, and bootstrapping compounds these challenges, often leading to instability and poor sample efficiency \citep{mnih2015human,sutton2018reinforcement,van2018deep}. One particularly critical consequence is \emph{overestimation bias} \citep{thrun1993issues}, where the estimated action-values ($Q$-values) exceed their true counterparts \citep{van2010double,van2016deep}. 

Overestimation bias in off-policy RL can arise from multiple sources. A well-studied source is the maximization operator in temporal-difference learning, which systematically drives estimates upward by preferring positive noises. It is typically addressed by min-clipping \citep{fujimoto2018addressing,haarnoja2018soft}, where the minimum of two critic networks is used, though this introduces a systematic underestimation bias \citep{van2010double}. Other contributors include sampling error in batch temporal-difference learning \citep{pavse2020reducing} and reuse bias, which arises when the same data are repeatedly used for updates \citep{ying2023reuse}.

Although the remedies reduce overestimation, they often swing too far toward pessimism, producing overly pessimistic policy evaluations discouraging exploration. This phenomenon, known as {\it pessimistic under-exploration} \citep{ciosek2019better}, limits effective learning in actor-critic methods, necessitating methods that can balance optimism and pessimism more systematically. However, the interaction between critic errors and resulting actor behavior follows complex dynamics in online deep RL \citep{moskovitz2021tactical}, and remains poorly understood theoretically.

We posit that the impact of the interplay between optimism and pessimism on the learning profiles of actor-critic pipelines can be explained by a simple and intuitive pair of hyperparameters. We introduce a new framework that puts the action-value predictions of a pair of critics into a utility-theoretic perspective, which considers optimism and pessimism as points along a spectrum.
This framework is practical as the hyperparameters controlling optimism and pessimism are interpretable and can be tuned for different tasks.
The commonplace min-clipping boils down to only one of infinitely many possibilities within this spectrum. The descriptive value of this spectrum can be of the interest to theoreticians. It can also be used by practitioners as a hyperparameter to improve performance via grid-search. 

\section{Prior art of orchestrating optimism and pessimism} \label{sec:motivation}
The state-of-the-art suggests the following three strategies to counteract the under-exploratory effects of pessimistic policy evaluation:

(i) \textbf{Maximum-entropy policy search.} Algorithms such as Soft Actor-Critic (SAC) \citep{haarnoja2018soft} maintain exploration by incentivizing an entropy increase in the policy distribution. This incentive is added to the reward function in the Bellman target calculation and is accounted for also during actor training. In effect, this is a random exploration scheme that perturbs the action selection independently from an estimation of the accuracy of the action-value approximations. This disconnect manifests itself as the brittleness of the algorithm on the entropy coefficient \citep{haarnoja2018softa}.

(ii) \textbf{Optimistic exploration.} Many theoretically-grounded approaches for model-driven exploration adopt the \emph{Optimism in the Face of Uncertainty (OFU)} principle  \citep{auer2002finite,auer2008near}. \citet{ciosek2019better} implemented this principle into a deep actor-critic pipeline with their seminal Optimistic Actor-Critic (OAC) algorithm. OAC adopts an optimistic approximate upper confidence bound on the $Q$-value function to enhance policy exploration while using a pessimistic lower-confidence bound for critic training. The algorithm is computationally expensive. Furthermore, the effectiveness of its exploration is limited by the approximation errors incurred during confidence set calculation.

(iii) \textbf{Dynamic exploration.} The balance of optimism and pessimism is nuanced, and it varies across different tasks and develops throughout the training process \citep{moskovitz2021tactical}. The Tactical Optimism and Pessimism (TOP) algorithm, introduced by \citet{moskovitz2021tactical}, addresses this variability by dynamically balancing optimism and pessimism multi-armed bandits \citep{bubeck2012regret,lattimore2020bandit}. The approach assumes a Gaussian distributed action-value function and it tunes the degree of exploration by playing a bandit game on a predetermined grid of quantiles. The TOP algorithm uses the same level of pessimism/optimism for both critic and actor training in order to maintain the feasibility of the bandit game, which limits the versatility of the algorithm. Furthermore, as the algorithm is designed upfront for on-the-fly tuning, it does not permit an interpretable analysis of pessimism and optimism in environments of interest.

We propose a simpler alternative: treating the critic pairs as samples from a parametric distribution over action-value functions. For example, when using a Laplace distribution as our main focus in the paper, embedding it within a utility function allows us to express the degree of critic pessimism and actor optimism on an interpretable scale between $-1$ and $+1$. We demonstrate the expressive and instrumental power of this framework within a SAC extension due to the wide adoption of the model family. We refer to the resulting algorithm as a Utility Soft Actor Critic (USAC). However, our framework is equally applicable to alternative families such as deterministic policy gradient variants \cite{lillicrap2016continuous, fujimoto2018addressing}. 
As shown in \Cref{fig:optimism_pessimism}, our framework reveals how the specific demands of an environment for pessimistic policy evaluation and optimistic policy search influence each other. The left panel illustrates the impact of different combinations of pessimism and optimism on final returns, where a coarse grid-search can yield improvements exceeding 30\% in several cases. The right panel shows the corresponding effect on value estimation bias (true value minus estimated value), indicating whether the policy tends to under- or over-estimate returns.

\begin{figure}[h!]
\centering
\includegraphics[width=0.49\columnwidth]{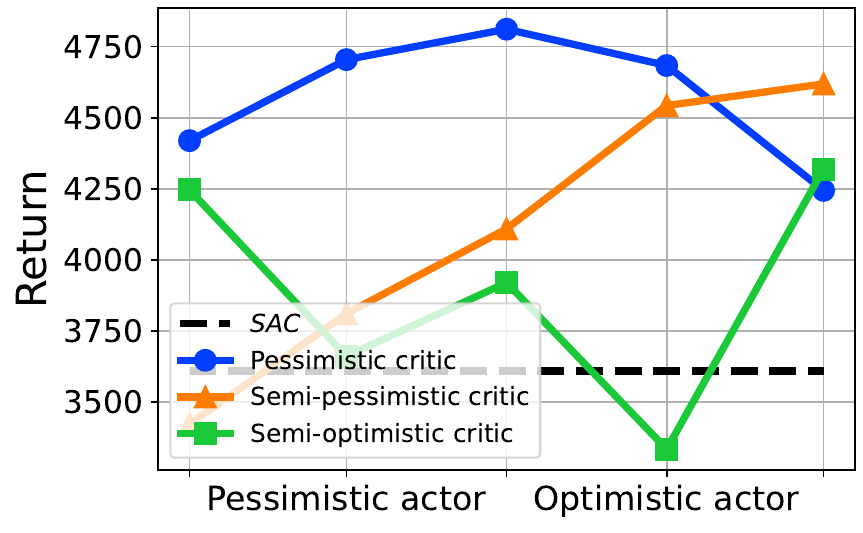}
\includegraphics[width=0.49\columnwidth]{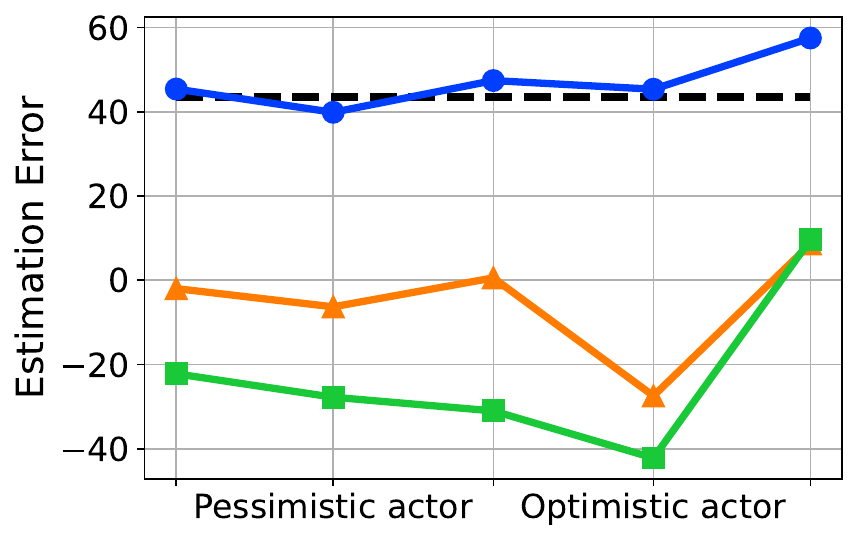}
\caption{Impact of optimism and pessimism levels on actor–critic performance on Walker2d environment. Left: Final returns for different combinations of pessimism and optimism. Right: Value estimation bias (i.e., true – estimated value); positive values indicate underestimation, negative values indicate overestimation}
\label{fig:optimism_pessimism}
\end{figure}

\section{Preliminaries} \label{sec:preliminary}

Throughout this paper, we denote $\mathcal{P}(\Omega)$ as the set of all probability distributions on $\Omega$ and let $\mathcal{B}(\Omega)$ be the set of bounded functions on $\Omega$.

\textbf{Markov decision processes (MDPs).}
An MDP can be represented by the tuple $\mathcal{M}=\langle\mathcal{S},\mathcal{A},p,p_{0},r,\gamma\rangle$ \citep{puterman2014markov}. Here, $\mathcal{S}$ and $\mathcal{A}$ denote the continuous state and action spaces. The function $p(s'|s,a)$ for $(s,a,s')\in \mathcal{S}\times\mathcal{A}\times\mathcal{S}$ represents the transition probability from the current state $s$ and action $a$ to the next state $s'$. We consider the model-free setup where this function is not known and will not be learned. The initial state distribution is denoted by $p_{0}\in\mathcal{P}(\mathcal{S})$, while $r:\mathcal{S}\times\mathcal{A}\rightarrow[0,B_{r}]$ represents the bounded reward function with $B_{r}>0$, and $\gamma\in[0,1]$ stands for the discount factor.

\textbf{Policies.}
Let $\Pi=\{\pi:\mathcal{S}\rightarrow\mathcal{P}(\mathcal{A})\}$ be the set of policies. The interaction between the agent and the MDP $\mathcal{M}$ under some policy $\pi\in\Pi$ progresses iteratively. At each time step $t\in\mathbb{N}$, the agent observes state $s_{t}\in\mathcal{S}$, chooses action $a_{t}\in\mathcal{A}$ based on the policy $a_{t}\sim\pi(\cdot\vert s_{t})$, receives a (bounded) reward $r_{t}:=r(s_{t},a_{t})$, and transitions to the next state $s_{t+1} \sim p(\cdot\vert s_{t},a_{t})$. Additionally, we define $p^{\pi}(s',a'\vert s,a)=p(s'\vert s,a)\pi(a' \vert s')$ as the one-step transition probability from $(s,a)$ to $(s',a')$.

\textbf{Maximum entropy RL.}
The standard goal of RL is to find a policy $\pi$ that maximizes the expected sum of discounted rewards $J(\pi) = \E_{\pi}[\sum_{t=0}^{\infty} \gamma^{t}r_{t}]$ with initial state $s_{0} \sim p_{0}$ \citep{sutton2018reinforcement,bertsekas1996neuro,szepesvari2022algorithms}. To improve exploration, we will consider the more general maximum entropy RL framework \citep{ziebart2010modeling,haarnoja2017reinforcement,haarnoja2018soft,schulman2017equivalence}. In this framework, one aims not only to maximize the expected (discounted) return but also to maximize the (discounted) entropy of the actions suggested by the stochastic policies. Formally, this is achieved by incorporating an entropy term, tempered by a parameter $\alpha>0$, which governs the relative importance of the return compared to the policy entropy. The maximum entropy objective function is defined as ${J_{\alpha}(\pi)=\E_\pi\lrs{\sum_{t=0}^\infty \gamma^t (r_{t}+\alpha \mathcal{H}(\pi(\cdot|s_t)))}}$, where the policy entropy term is defined as $\mathcal{H}(\pi(\cdot|s)) = -\E_{a\sim\pi(\cdot\vert s)}[\log\pi(a|s)]$.

A policy is typically assessed by an action-value function ${Q:\mathcal{S}\times \mathcal{A}\rightarrow \R}$ \citep{watkins1992q}. Specifically, starting from initial state $s_{0}=s$ and action $a_{0}=a$, the action-value function $Q^\pi$ obtained by following policy $\pi$ satisfies $Q^{\pi}(s,a)=J_{\alpha}(\pi)$. In addition, this is the unique fixed-point solution to the soft Bellman operator ${T^\pi:\cal B(\cal S\times \cal A)\rightarrow  B(S\times A)}$ \citep{bertsekas1996neuro}:
\begin{equation} 
\resizebox{1.0\hsize}{!}{\ensuremath{T^\pi Q(s,a) = r(s,a) + \gamma\E_{(s',a')\sim p^\pi(\cdot \vert s,a)}[Q(s',a') -\alpha\log\pi(a'|s')]}}, \label{eq:sac:bellmann_op}
\end{equation}
that is, $T^\pi Q^\pi(s,a) = Q^\pi(s,a)$ for any $(s,a)\in\cS\times \cA$.

\textbf{Actor-critic algorithms.}
Performing continuous control with deep RL necessitates a policy-gradient approach applicable to online updates. The actor-critic design is an optimal and widely adopted framework for this purpose. In these algorithms, an agent simultaneously learns a policy (the actor), responsible for selecting actions that maximize expected return, and an estimation of the action-value function (the critic), responsible for evaluating the policy's quality through iterative updates. Formally, the \emph{critic} $Q$ is trained to predict the action-value function $Q^\pi$ (derived by following policy $\pi$) by temporal difference learning
\begin{equation} \label{eq:critic-training}
    \argmin_{Q} \E_{(s,a,r,s')\sim \mathcal{D}}[(Q(s,a) - y(s,a,r,s'))^{2}],
\end{equation}
where $(s,a,r,s')$ is sampled from the replay buffer $\mathcal{D}$, and the critic's \emph{target value} is defined as
\begin{equation}\label{eq:critic-target:plain}
    y(s,a,r,s') = r + \gamma[\bar{Q}(s',a') - \alpha\log\pi(a'|s')], \quad a'\sim \pi(\cdot|s'),
\end{equation}
which is built using a \emph{target critic} $\bar{Q}$  that may differ from the critic $Q$ \citep{mnih2013playing,mnih2015human}. The \emph{actor} learns its policy by solving the following:
\begin{equation} \label{eq:actor-training}
    \argmax_{\pi\in\Pi} \E_{a\sim \pi(\cdot|s)}[\tilde{Q}(s,a) - \alpha \log \pi(a|s)],
\end{equation}
where $\tilde{Q}$ is (again) an estimate of the action-value function $Q^{\pi}$, which can differ from both the critic $Q$ and the target critic $\bar{Q}$. In scenarios where the agent does a full Bellman backup without approximation errors, the critic $Q$, the target critic $\bar{Q}$, and the critic used to update the actor $\tilde{Q}$ are all identical to $Q^{\pi}$ \citep{sutton2018reinforcement,haarnoja2018soft}. However, variations in $Q$, $\bar{Q}$, and $\tilde{Q}$ are often introduced to address the challenges outlined in the introduction. We will explore some of these variations in more detail below.

\textbf{The sources of optimism and pessimism in actor-critics.}  A common strategy to address overestimation involves estimating the action-value function using twin critics, along with associated target critics that are delayed versions of these critics \citep{lillicrap2016continuous,mnih2013playing,mnih2015human,fujimoto2018addressing}.
Specifically, let $Q_1,Q_2:\mathcal{S}\times \mathcal{A}\rightarrow \R$ be the two critics and $\bar{Q}_1,\bar{Q}_2:\mathcal{S}\times \mathcal{A}\rightarrow \R$ the two target critics that are delayed versions of the critics. The pessimistic strategy, employed by TD3 \citep{fujimoto2018addressing}, SAC \citep{haarnoja2018soft} and OAC \citep{ciosek2019better}, is to use the minimum of the two target critics $\{\bar Q_k\}_{k=1,2}$ to compute the target value of the critics in \eqref{eq:critic-target:plain}: 
\begin{equation} \label{eq:critic-target:TD3}
\bar{Q}(s,a) = \min_{k=1,2} \bar{Q}_{k}(s,a).
\end{equation}
This target is used to train both critics $Q_1$ and $Q_2$. Similarly, the same approach is employed to update the actor (in TD3 and SAC). Here, the critic used to update the actor, $\tilde{Q}$ in \eqref{eq:actor-training}, is determined by the minimum of the two critics $\{Q_{k}\}_{k=1,2}$:
\begin{equation} \label{eq:actor-target:SAC}
\tilde{Q}(s,a) = \min_{k=1,2}Q_{k}(s,a).
\end{equation}
\citet{ciosek2019better} noted that the pessimistic approach, described in \cref{eq:critic-target:TD3,eq:actor-target:SAC}, can lead to pessimistic under-exploration. To prevent this, they proposed embracing the \textit{optimism in the face of uncertainty} principle in their OAC algorithm. However, this approach can potentially make exploration too aggressive, thereby negatively impacting overall performance \citep{moskovitz2021tactical}. Instead, the balance between optimism and pessimism should be tailored to the problem's nature. \citet{moskovitz2021tactical} suggested using the following target critic to achieve this balance:
\begin{equation}\label{eq:critic-target:TOP}
    \bar{Q}(s,a)=\bar{\mu}(s,a) + \beta \bar{\sigma}(s,a), \quad \beta\in\R,
\end{equation}
where the mean $\bar{\mu}$ and (unbiased) standard deviation $\bar{\sigma}$ are constructed using the two target critics $\bar Q_1, \bar Q_2$:
\begin{align*}
  \bar{\mu}(s,a) = \frac{1}{2}(\bar{Q}_1(s,a) + \bar{Q}_2(s,a))   
\end{align*}
and 
\begin{align*}
\bar{\sigma}(s,a) = \sqrt{\sum_{k=1,2}(\bar{Q}_{k}(s,a) - \bar{\mu}(s,a) )^2}.    
\end{align*}
The critic to update the actor, $\tilde Q$, can be defined in a similar manner from the two critics $\{Q_k\}_{k=1,2}$. \citet{moskovitz2021tactical} classify \cref{eq:critic-target:TOP} as \emph{optimistic} when $\beta\ge 0$ and \emph{pessimistic} when $\beta<0$. In the special case when $\beta=-1/\sqrt{2}$, \cref{eq:critic-target:TOP} simplifies to taking the minimum of the two critics  as in \cref{eq:critic-target:TD3,eq:actor-target:SAC}. The TOP algorithm tries to adjust the balance between optimism and pessimism by using a multi-armed bandit algorithm to select $\beta$ from the set $\{-1,0\}$; here $\beta=-1$ corresponds to a pessimistic estimate and $\beta=0$ to an optimistic one. Although $\beta$ can vary over time using a bandit algorithm, TOP uses the same $\beta$ for both $\bar{Q}$ and $\tilde{Q}$.

\section{A utility-theoretic framework to make the action-value biases steerable} \label{sec:USAC}

Our key contribution is a new theoretical framework to quantify the degree of optimism or pessimism assumed by an actor-critic algorithm that performs policy evaluation via twin critics using two interpretable scalars, one for the critic and one for the actor. We achieve this by assuming that the action-value function follows a parametric distribution and the twin critics are two samples from this distribution. The resulting utility function yields a form that enables describing the optimism and pessimism levels via scalar risk coefficients that take the commonplace min-clipping method as a special case and provides an effortless generalization of it by simple grid-search.

Let $\mathcal{Q}\in \mathcal{P}(\mathcal{B}(\mathcal{S}\times \mathcal{A}))$ be a distribution over $Q$-functions, serving as the agent's estimate of the $Q$-functions. The dispersion of the distribution reflects the uncertainty of the learned $Q$-functions. To characterize such a distribution, we use the utility function of $\mathcal{Q}$, defined by
\begin{equation} \label{eq:utility_function}
    U^{\mathcal{Q}}_{\lambda} (s,a) = \frac{1}{\lambda}\log\E_{Q\sim\mathcal{Q}}[\exp(\lambda Q(s,a))], \quad \lambda\in\R,
\end{equation}
where $U^{\mathcal{Q}}_{\lambda}(s,a)=\E_{Q\sim\mathcal{Q}}[Q(s,a)]$ when $\lambda=0$. This utility function is commonly used for measuring risk in finance, economics, and decision-making under uncertainty \citep{follmer2011stochastic,von1947theory}. It also shares similarities with the utility functions employed in risk-sensitive RL \citep{howard1972risk,shen2014risk,prashanth2022risk} and distributional RL \citep{rowland2019statistics,bellemare2017distributional,bellemare2023distributional}. It is noteworthy that we apply this utility function to the $\mathcal{Q}$ distribution rather than the return distribution, as is typically done in risk-sensitive and distributional RL.
The utility function is also referred to as \emph{expected utility}, \emph{exponential utility}, \emph{exponential risk measure}, \emph{generalized mean}, or \emph{entropic risk measure} according to the context \citep{follmer2011entropic,follmer2011stochastic}. 

We use \cref{eq:utility_function} as a measure of pessimism or optimism of the $\mathcal{Q}$ distribution: a positive value of $\lambda$ implies an \emph{optimistic} view, while a negative value indicates a \emph{pessimistic} view. This interpretation is clear from a second-order Taylor expansion applied to the utility function: 
\begin{align*}
U^{\mathcal{Q}}_{\lambda}=\E_{Q\sim\mathcal{Q}}[Q]+\frac{\lambda}{2}\V_{Q\sim\mathcal{Q}}[Q]+\mathcal{O}(\lambda^{2}).    
\end{align*}
Consequently, the utility function $U^{\mathcal{Q}}_{\lambda}$ spans the entire spectrum of the $\mathcal{Q}$ distribution by varying $\lambda$. 

\subsection{Policy evaluation, improvement and guarantees} \label{sec:USAC:prop}

For any given policy $\pi\in\Pi$, the utility $U^\cQ_\lambda\in\mathcal{B}(\cS\times\cA)$ will reduce to $Q^\pi$ for any $|\lambda|<\infty$, when $\cQ$ is a Dirac delta distribution centered at $Q^\pi$. This implies that the deviation of $U^\cQ_\lambda$ from $Q^\pi$ is due to the uncertainty in the learned $Q$-function. In fact, we observe that the Bellman operator  should also bring the utility $U^\cQ_\lambda$ to its fixed point $Q^\pi$~\citep{bertsekas1996neuro}
\begin{lemma}\label{lem:BE2TD}
For any expected utility of $\cQ \in \mathcal{P}(\cS \times \cA)$, we have
\begin{equation*}%
\lVert T^\pi U^{\mathcal{Q}}_\lambda-U^{\mathcal{Q}}_\lambda \rVert^2_{\rho^\pi}\leq \E_{(s,a)\sim\rho^\pi}\E_{(s',a')\sim p^\pi(\cdot|s,a)}\lrs{\ell(U^{\cQ}_\lambda;\xi)},
\end{equation*}
where 
\begin{align*}
\ell(U^{\cQ}_\lambda;\xi)=(r(s,a)+\gamma (U^{\cQ}_\lambda(s',a')-\alpha\log(\pi(a'|s')))-U^{\cQ}_\lambda(s,a))^2    
\end{align*}
is the squared loss of $U^\cQ_\lambda$ on an observation $\xi=(s,a,s',a')$.
\end{lemma}
\begin{proof}
For any $\cQ \in \mathcal{P}(\mathcal{S}\times\mathcal{A})$, we have
\begin{align*}
&\lVert T^{\pi}U^{\cQ}_\lambda - U^{\cQ}_\lambda \rVert_{\rho^{\pi}}^{2} = \mathbb{E}_{(s,a)\sim\rho^\pi}[(T^{\pi}U^{\cQ}_\lambda(s,a)-U^{\cQ}_\lambda(s,a))^2] \\
&= \mathbb{E}_{(s,a)\sim \rho^\pi}[(r(s,a)+\gamma \mathbb{E}_{(s',a')\sim p^\pi(\cdot|s,a)}[U^{\cQ}_\lambda(s',a')- \\ & \quad \quad \quad \quad \quad \quad \alpha\log(\pi(a'|s'))]-U^{\cQ}_\lambda(s,a))^2] \\
&\leq \mathbb{E}_{(s,a)\sim \rho^\pi}\mathbb{E}_{(s',a')\sim p^\pi(\cdot|s,a)}[(r(s,a)+\gamma (U^{\cQ}_\lambda(s',a')-\\&
\quad \quad \quad \quad \quad \quad
\alpha\log(\pi(a'|s')))-U^{\cQ}_\lambda(s,a))^2],
\end{align*}
where the inequality comes from the Jensen’s inequality.
\end{proof}

Thus, as a consequence of \cref{lem:BE2TD}, the Bellman error  defined by
\begin{align*} %
&T^\pi U^{\mathcal{Q}}_\lambda(s,a)-U^{\mathcal{Q}}_\lambda(s,a)=r(s,a)\\
&\qquad +\gamma\E_{(s',a')\sim p^\pi(\cdot|s,a)}\lrs{U_\lambda^{\mathcal{Q}}(s',a')-\alpha\log(\pi(a'|s'))} - U_\lambda^{\mathcal{Q}}(s,a)
\end{align*}
should diminish. Nevertheless, it fosters exploration during learning by maintaining a distribution over the $Q$-functions.  On the other hand, given a $\cQ$ distribution associated with the current policy $\pi$, denoted as $\cQ^\pi$, we update the policy, similar to \cref{eq:actor-training}, through the following maximization problem:
\begin{equation*}
    \argmax_{\pi\in\Pi}\E_{a\sim \pi(\cdot|s)}[ U_{\lambda}^{\cQ^{\pi}}(s,a) - \alpha \log \pi(a|s)].
\end{equation*}
Extending the policy iteration process from SAC to utility functions ensures that the guarantees of policy improvement still hold \citep{haarnoja2018soft,haarnoja2017reinforcement}. 

\subsection{The Utility Soft Actor Critic (USAC)} \label{sec:USAC:alg}

We illustrate a simple implementation of our utility-theoretic formulation into a soft actor-critic pipeline and refer to the resulting algorithm as the {\it Utility Soft Actor Critic (USAC)}. We perform online and off-policy learning as usual, maintaining a replay buffer of observed environment interactions $(s,a,r,s')\sim\cD$. As the operator $T^\pi$ should bring the utility $U^\cQ_\lambda$ to its fixed-point $Q^\pi$, the USAC framework learns the critic $Q$ (an estimate of $Q^\pi$) as in \cref{eq:critic-training} but with the critic's \emph{target value} defined as
\begin{equation}\label{eq:critic-target:USAC}
    y(s,a,r,s') = r + \gamma[U_{\lambda_{\text{critic}}}^{\bar{\cQ}}(s',a') - \alpha\log\pi(a'|s')],
\end{equation}
where $a'\sim \pi(\cdot|s')$. Note that we use a \emph{target critic distribution} $\bar{\cQ}$ with some utility parameter $\lambda_{\text{critic}}$ chosen specifically for the critic. Similar to \cref{eq:actor-training}, the \emph{actor} learns its policy by
\begin{equation} \label{eq:actor-training:USAC}
    \argmax_{\pi\in\Pi} \E_{a\sim \pi(\cdot|s)}[U_{\lambda_{\text{actor}}}^{\tilde\cQ}(s,a) - \alpha \log \pi(a|s)],
\end{equation}
using $\tilde{\cQ}$ as the estimate for $\cQ$ and $\lambda_{\text{actor}}$ as the utility parameter for the actor. 

Note that \citet{moskovitz2021tactical} implements a single shared optimism/pessimism parameter for the actor and the critic. In contrast, our framework naturally decouples the optimism and pessimism levels of the actor and critic into two separate scalars, $\lambda_{\text{critic}}$  and  $\lambda_{\text{actor}}$, respectively. This decoupling allows both investigating the relationship between the optimism and pessimism demands of the actor and the critic and tuning them separately, thereby benefiting from their complementary roles in online training. Later in the experiments presented in \cref{sec:experiments}, we will explore various combinations of optimism and pessimism levels. 
\begin{algorithm}[t!]
  \caption{USAC with Laplace distribution (\cref{sec:USAC-lap})}
  \label{alg:USAC}
  \begin{algorithmic}[1]
    \Require Averaging parameter $\tau\in(0,1)$, learning rates $\eta_{\theta},\eta_{\phi}>0$, mini-batch size $n\in\mathbb{N}$, pessimism/optimism coefficients $\kappa_{\text{actor}},\kappa_{\text{critic}}\in(-1,1)$, entropy coefficient $\alpha$, \\ Polyak coefficient $\tau$
    \State \textbf{Initialize}: Replay buffer $\mathcal{D}=\emptyset$, initial state $s_0\sim p_0$, critic $\{\theta_{k}\}_{k=1,2}$, target critic $\{\bar{\theta}_{k}\}_{k=1,2}$ and actor policy $\phi$
    \For{each time step}
    \For{each environment step}
    \State $a_t\sim \pi_\phi(\cdot|s_t)$ \Comment{Sample action from the policy $\phi$}
    \State $s_{t+1}\sim p(\cdot|s_t,a_t)$ \Comment{Sample transition from the environment}
    \State $\mathcal{D}\leftarrow\mathcal{D} \cup (s_t,a_t,r_t,s_{t+1})$ \Comment{Store transition in the replay pool}
    \EndFor
    \For{each training step}
    \State $\{(s_i,a_i,r_i,s_i',a_i') : (s_i,a_i,r_i,s_i')\sim \mathcal{D},a_i'\sim \pi_\phi(\cdot|s_i')\}_{i=1}^{n}$ \Comment{Sample mini-batch}
    \For{$i\in\{1,\dots,n\}$} \Comment{Compute target critic dist. $\bar{\mathcal{Q}}$}
    \State $\mu_{\bar{\cQ}}(s_{i},a_{i}) \gets \frac{1}{2}(Q_{\bar{\theta}_{1}}(s_{i},a_{i}) + Q_{\bar{\theta}_{2}}(s_{i},a_{i}))$  
    \State $\sigma_{\bar{\cQ}}(s_{i},a_{i}) \gets \frac{1}{2} \lvert Q_{\bar{\theta}_{1}}(s_{i},a_{i}) - Q_{\bar{\theta}_{2}}(s_{i},a_{i})\rvert$
    \EndFor
    \State
    $y(s_{i},a_{i},r_{i},s_{i}') \gets$ \Comment{Compute target} \\$r_{i} + \gamma[\mu_{\bar{\cQ}}(s_{i}',a_{i}') + g(\kappa_{\text{critic}})\sigma_{\bar{\cQ}}(s_{i}',a_{i}') - \alpha\log\pi_{\phi}(a_{i}' \vert s_{i}')]$ 
    \For{each critic, $k\in\{1,2\}$} \Comment{Update critics}
    \State $\theta_k \gets \theta_k - \eta_{\theta} \nabla_{\theta_k}\{\frac{1}{n}\sum_{i=1}^{n} ( Q_{\theta_{k}}(s_{i},a_{i})-
    y(s_{i},a_{i},r_{i},s_{i}'))^{2}\}$
    \EndFor
    \For{$i\in\{1,\dots,n\}$} \Comment{Compute actor dist. $\tilde{\mathcal{Q}}$}
    \State $\mu_{\tilde{\cQ}}(s_{i},a_{i}) \gets \frac{1}{2}(Q_{\theta_{1}}(s_{i},a_{i}) + Q_{\theta_{2}}(s_{i},a_{i}))$  
    \State $\sigma_{\tilde{\cQ}}(s_{i},a_{i}) \gets \frac{1}{2} \lvert Q_{\theta_{1}}(s_{i},a_{i}) - Q_{\theta_{2}}(s_{i},a_{i})\rvert$
    \EndFor
    \State$\phi\gets\phi+\eta_{\phi}\nabla_{\phi}\{\frac{1}{n}\sum_{i=1}^{n} ( \mu_{\tilde\cQ}(s_{i},a_{i}) + g(\kappa_{\text{actor}}) \sigma_{\tilde\cQ}(s_{i},a_{i}) - \alpha \log \pi_{\phi}(a_{i} \vert s_{i}))\}$ \Comment{Update actor}
    \For{each critic, $k\in\{1,2\}$} 
    \State $\bar{\theta}_{k}\leftarrow \tau \theta_{k} + (1-\tau)\bar{\theta}_{k}$ \Comment{Update target critic}
    \EndFor
    \EndFor
    \EndFor
  \end{algorithmic}
\end{algorithm}

\subsection{Laplace distributed critics}\label{sec:USAC-lap}

The utility formulation in \cref{eq:utility_function} enables us to compute the utility for any distribution  with a moment-generating function. We will next present an implementation of a specific case of our USAC algorithm where $\cQ$ follows a Laplace distribution, which has the desirable property of expressing the utility parameter within the $[-1,+1]$. The pseudo-code of the resulting USAC algorithm is provided in \cref{alg:USAC}. Here, we denote the utility parameter associated with this distribution as $\kappa$, reflecting its characteristics within the Laplace context.
\begin{proposition} \label{prop:utility_function:laplace}
Suppose $\cQ$ is a Laplace distribution. Then, for any $(s,a)\in\mathcal{S}\times\mathcal{A}$, the utility function of $\cQ$ can be expressed as
\begin{equation} \label{eq:utility_function:laplace}
U^{\mathcal{Q}}_{\kappa}(s,a)=\mu_{\mathcal{Q}}(s,a)+g(\kappa)\sigma_{\mathcal{Q}}(s,a) \quad \text{for} \quad \kappa\in(-1,1),
\end{equation}
where $\mu_{\mathcal{Q}}$ and $\sigma_{\mathcal{Q}}$ represent the mean and standard deviation of $\mathcal{Q}$, respectively, and
\begin{equation} \label{eq:g_kappa}
 g(\kappa) = 
    \begin{cases}
        \log(1/(1-\kappa^{2}))/\sqrt{2}\kappa & \: \text{for} \: \kappa\in(-1,1) \backslash \{0\}, \\
        0 & \: \text{for} \: \kappa=0.
    \end{cases}
\end{equation}  
\end{proposition}
\begin{proof}
Suppose $Q(s,a)\sim \Lap(\mu_{\mathcal{Q}}(s,a), b_{\mathcal{Q}}(s,a))$. Then, for any $(s,a)\in\mathcal{S}\times\mathcal{A}$, the expected utility function of $\cQ$ on $(s,a)$ is
\begin{equation*}
U^{\mathcal{Q}}_{\lambda}(s,a)=\mu_{\mathcal{Q}}(s,a)-\lambda^{-1}\log(1-\lambda^{2}b_{\mathcal{Q}}^2(s,a)),
\end{equation*}
with $\vert\lambda\vert<1/b_{\mathcal{Q}}(s,a)$. Next, given that the variance of a Laplace distribution is $2b_\cQ^{2}$, we can rewrite the expected utility function as
\begin{equation*}
U^{\mathcal{Q}}_{\lambda}(s,a)=\mu_{\mathcal{Q}}(s,a)-\lambda^{-1}\log(1-\lambda^{2}\sigma_{\mathcal{Q}}^{2}(s,a)/2),
\end{equation*}
with $\vert\lambda\vert<\sqrt{2}/\sigma_{\mathcal{Q}}(s,a)$. Substituting $\lambda=\sqrt{2}\kappa/\sigma_{\mathcal{Q}}(s,a)$ for some $\kappa\in(-1,1)$, we have
\begin{equation*}
U^{\mathcal{Q}}_{\kappa}(s,a)=\mu_{\mathcal{Q}}(s,a)-\sigma_{\mathcal{Q}}(s,a)\log(1-\kappa^{2})/\sqrt{2}\kappa.
\end{equation*}
At last, this expression can be further simplified using the definition of $g(\kappa)$ in \cref{eq:g_kappa}.
\end{proof}
In \cref{prop:utility_function:laplace}, we observe that by varying $\kappa$ in the small interval $(-1,1)$, $U^{\mathcal{Q}}_{\kappa}$ spans the entire spectrum of the $\mathcal{Q}$ distribution. Specifically, as $\lim_{\kappa\rightarrow-1} g(\kappa) \rightarrow -\infty$, $\lim_{\kappa\rightarrow0}g(\kappa)\rightarrow 0$, and $\lim_{\kappa\rightarrow1} g(\kappa) \rightarrow \infty$. Alternatively, one could also presume Gaussian-distributed $Q$-values.\begin{proposition} \label{prop:utility_function:gaussian}
Suppose $\cQ$ is Gaussian with  mean $\mu_{\mathcal{Q}}$ and variance $\sigma_{\mathcal{Q}}^{2}$, i.e., $Q(s,a)\sim \mathcal{N}(\mu_{\mathcal{Q}}(s,a), \sigma_{\mathcal{Q}}^{2}(s,a))$. Then, for any ${(s,a)\in\mathcal{S}\times\mathcal{A}}$, the utility function of $\cQ$ can be expressed as
\begin{equation} \label{eq:utility_function:gaussian}
U^{\mathcal{Q}}_{\lambda}(s,a)=\mu_{\mathcal{Q}}(s,a)+\lambda\sigma_{\mathcal{Q}}^{2}(s,a)/2 \quad \text{for} \quad \lambda\in(-\infty,\infty).
\end{equation}
\end{proposition}
\begin{proof}[Proof of \cref{prop:utility_function:gaussian}]
The proof follows by similar steps as in the proof of \cref{prop:utility_function:laplace}.
\end{proof}

Now suppose the uncertainty follows a Laplace distribution. Following the algorithmic steps outlined in \cref{sec:USAC:alg}, we can estimate the utility of the target critic distribution $\bar{\cQ}$ by two target critics $\{\bar{Q}_{k}\}_{k=1,2}$, which are delayed versions of the critics $\{Q_{k}\}_{k=1,2}$. The utility function is given by $U^{\bar{\cQ}}_{\kappa_{\text{critic}}} = \mu_{\bar{\cQ}} + g(\kappa_{\text{critic}})\sigma_{\bar{\cQ}}$ with utility parameter $\kappa_{\text{critic}}\in(-1,1)$, where
\begin{equation} \label{critic-target:USAC:Lap:mean}
    \mu_{\bar{\cQ}}(s,a) = \frac{1}{2}(\bar{Q}_1(s,a) + \bar{Q}_2(s,a) ),
\end{equation}
and
\begin{equation} \label{critic-target:USAC:Lap:std}
    \sigma_{\bar{\cQ}}(s,a) =  \frac{1}{2} \lvert \bar{Q}_{1}(s,a) - \bar{Q}_{2}(s,a) \rvert.
\end{equation}
Hence, the target value $y(s,a,r,s')$ to train the critic becomes
\begin{equation} \label{eq:critic-target:lap:plain}
    r + \gamma[\mu_{\bar{\cQ}}(s',a') + g(\kappa_{\text{critic}})\sigma_{\bar{\cQ}}(s',a') - \alpha\log\pi(a'|s')],
\end{equation}
with $a'\sim \pi(\cdot|s')$.
Similarly, the actor learns its policy by
\begin{equation} \label{eq:actor-training:USAC:Lap}
    \argmax_{\pi\in\Pi} \E_{a\sim \pi(\cdot|s)}[\mu_{\tilde\cQ}(s,a) + g(\kappa_{\text{actor}}) \sigma_{\tilde\cQ}(s,a) - \alpha \log \pi(a|s)], 
\end{equation}
for $\kappa_{\text{actor}}\in(-1,1)$, where $\mu_{\tilde{\cQ}}$ and $\sigma_{\tilde{\cQ}}$ are estimated in similar manner as \cref{critic-target:USAC:Lap:mean,critic-target:USAC:Lap:std} on $\tilde{\cQ}$ using the critics $\{Q_k\}_{k=1,2}$. The pseudo-code of this algorithm can be found in \cref{alg:USAC}.

We note that a smaller $\kappa$ value indicates more pessimism, while a larger $\kappa$ value indicates more optimism. By choosing $\kappa$ such that $g(\kappa)=-1$ (i.e., $\kappa\approx-0.83$), the utility in  \cref{eq:utility_function:laplace} simply reduces to the minimum of the two critics. 
Thus, choosing $\kappa_{\text{critic}}=\kappa_{\text{actor}}\approx -0.83$, we obtain the pessimism choices of SAC \citep{haarnoja2018soft} and TD3 \citep{fujimoto2018addressing}. On the other hand, by choosing $\kappa_{\text{critic}}\approx -0.83$ and a positive $\kappa_{\text{actor}}$, we have a pessimism critic and an optimistic actor as explored in OAC \citep{ciosek2019better}. Lastly, by letting a multi-armed bandit algorithm choose $\kappa_{\text{critic}}=\kappa_{\text{actor}}\in\{-0.916563,0\}$, we recover the pessimistic/optimistic design suggested by TOP \citep{moskovitz2021tactical}. To recover the pessimistic case of the TOP algorithm, we can choose $\kappa=-0.916563$, which corresponds to having the function $g$ in \cref{eq:g_kappa} equal to $\sqrt{2}$. We note that by choosing $-0.916563$, TOP considers a more pessimistic critic target than SAC, TD3, and OAC.

\section{Experiments} \label{sec:experiments}
We conduct experiments to answer two questions: (Q1) Can our utility-theoretic framework explain the previously observed relationships between the action-value biases and exploration optimism by two simple scalars, one for actor and one for the critic? (Q2) Can these scalars be steered to improve  performance via a coarse grid-search?

Throughout the experiments, we adopted the USAC version described in \cref{alg:USAC} that assumes Laplace-distributed critics as the representative implementation of our framework. We tested all models on the standard five continuous control benchmarks of the MuJoCo physics engine \citep{todorov2012mujoco,brockman2016openai} to facilitate the interpretation of our results and its comparability to previous studies. The implementation was done in \texttt{PyTorch} \citep[Version $2.1.0$]{paszke2019pytorch}.
We ran each experiment for one million time steps and tested each policy every $10,000$ steps over ten evaluation episodes. We repeated each experiment across five different seeds hard-set to numbers 1 to 5 for training. We chose seed numbers  for evaluation as the training repetition number plus 100. We consider SAC \citep{haarnoja2018soft}, TD3 \citep{fujimoto2018addressing}, OAC \citep{ciosek2019better}, and TOP \citep{moskovitz2021tactical} as baselines, as they represent the most powerful algorithms that take finding an optimal balance between optimism and pessimism as the main goal.
\footnote{For TOP, we use \texttt{\url{https://github.com/tedmoskovitz/TOP}} with the number of random actions set to 10,000.}

We modeled all the actor and critic networks as fully connected layers of $256$ hidden units that follow \texttt{ReLU()} activations.  USAC, SAC, and OAC actors output parameters for a \texttt{SquashedGaussian} module via a linear layer of size \(2 \times d_a\), where the first outputs represent the mean and the second represent the variance. TD3 and TOP use deterministic policies with a linear layer of size \(d_a\) followed by \texttt{Tanh()}. Critic networks for USAC, SAC, OAC, and TD3 output scalars; TOP outputs $N=50$ quantiles. Here, \(d_s\) and \(d_a\) denote state and action dimensions.
Note that although the paper of TOP describes a two-layer network, its code uses three layers; we modify it for consistency with other baselines. For SAC and USAC, we employ automatic entropy temperature tuning $\alpha$ \citep{haarnoja2018softa}. For OAC and TOP, we adopt the best parameters reported in their respective papers.  The hyperparameters and network configurations shared by all models in our consideration are summarized in \cref{tab:hyperparameters}. The baseline models do not have any additional hyperparameters required to be tuned. We provide a public implementation at https://github.com/adinlab/UtilitySoftActorCritic.

\begin{table}[htb]
\centering
\caption{Shared hyperparameters }
\label{tab:hyperparameters}
\begin{tabular}{lcc}
\\ \toprule
Hyperparameter & Value \\
\midrule
Evaluation episodes &10\\
Evaluation frequency & Maximum time steps / 100\\
Discount factor $(\gamma)$ & 0.99\\
$n$-step returns &1 step\\
Replay ratio/update-to-data &1\\
Replay buffer size &1,000,000\\
Maximum time steps &1,000,000\\ 
Mini-batch size $(n)$ & 256\\
Number of critics & 2 \\
Actor/critic optimizer & Adam \citep{kingma2014adam} \\
Optimizer learning rates $(\eta_{\phi},\eta_{\theta},\eta_{\alpha})$ &3e-4\\
Averaging parameter $(\tau)$ &5e-3\\
\bottomrule
\end{tabular}
\end{table}

Our framework can be implemented as a drop-in into any standard actor-critic pipeline as an alternative way to aggregate critic outputs into a single value. Hence, it does not cause a tangible computational overhead. Our representative USAC implementation adopts the twin-critic and single-actor architecture used in standard SAC, TD3, OAC, and TOP implementations. Its computation time is indistinguishable from a standard SAC implementation.

\begin{figure*}[htb]
\centering
\includegraphics[width=0.195\textwidth]{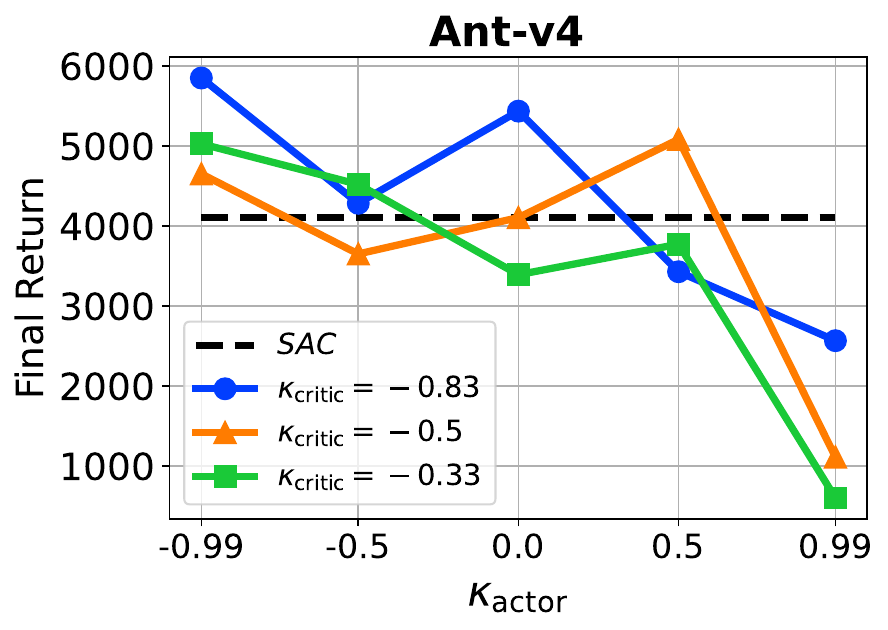}
\includegraphics[width=0.195\textwidth]{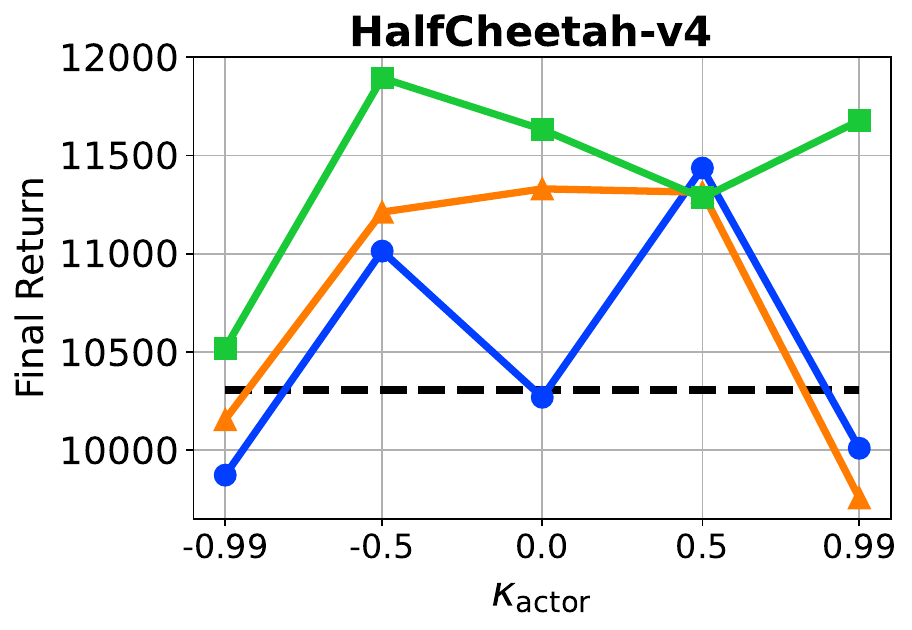}
\includegraphics[width=0.195\textwidth]{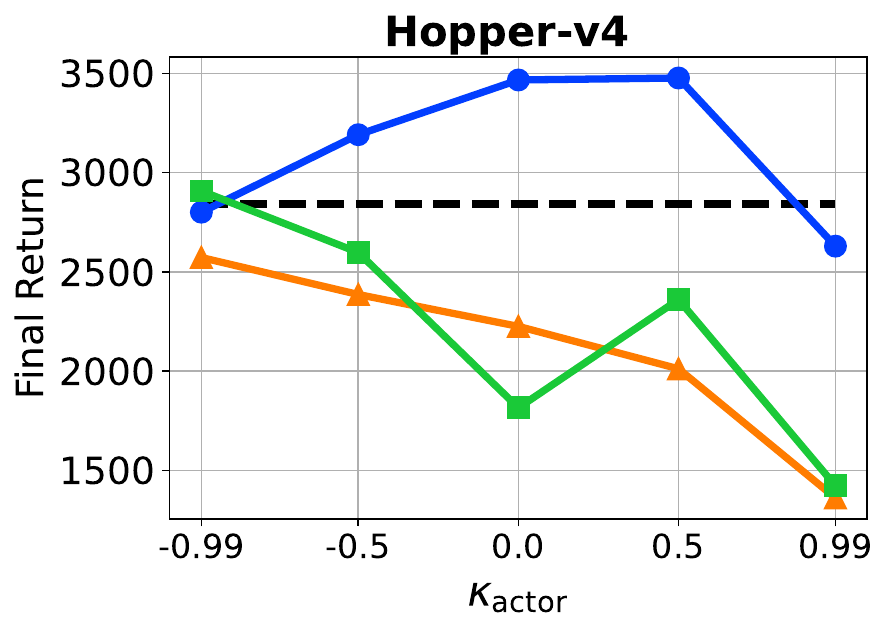}
\includegraphics[width=0.195\textwidth]{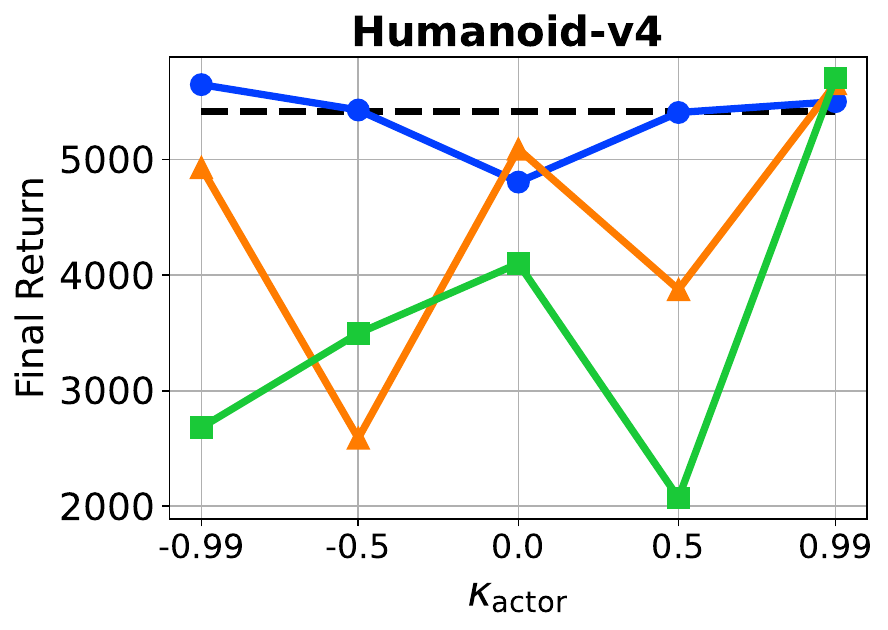}
\includegraphics[width=0.195\textwidth]{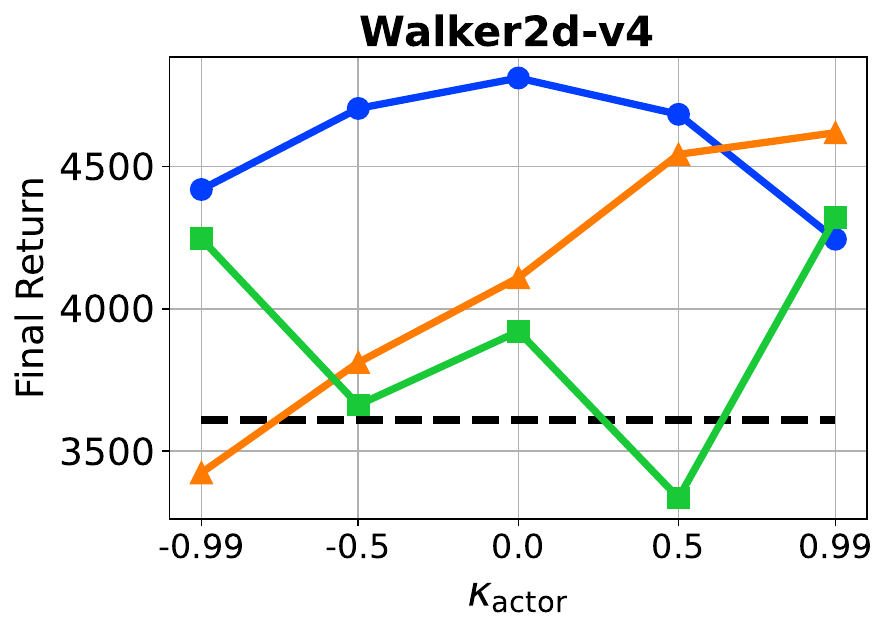}
\caption{The impact of the degree of actor optimism on the final return of USAC algorithm when Laplace distributed critics are assumed. The results are averaged over three seeds, each with three evaluation episodes. The experiments span various $\kappa_{\text{critic}}$ and $\kappa_{\text{actor}}$ values, using SAC as the baseline.}
 \label{fig:grid_rewards}
\end{figure*}
\begin{figure*}[htb]
\centering
\includegraphics[width=0.195\textwidth]{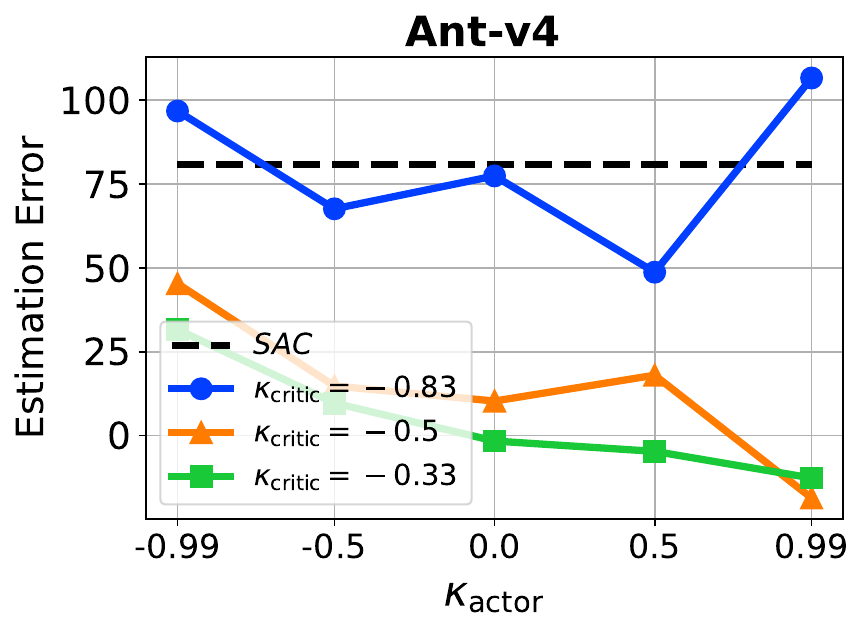}
\includegraphics[width=0.195\textwidth]{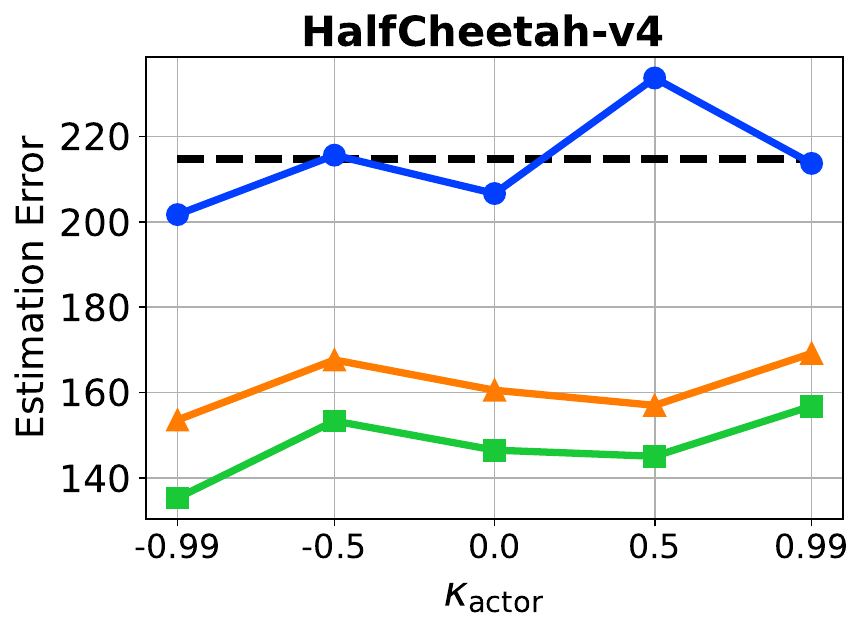}
\includegraphics[width=0.195\textwidth]{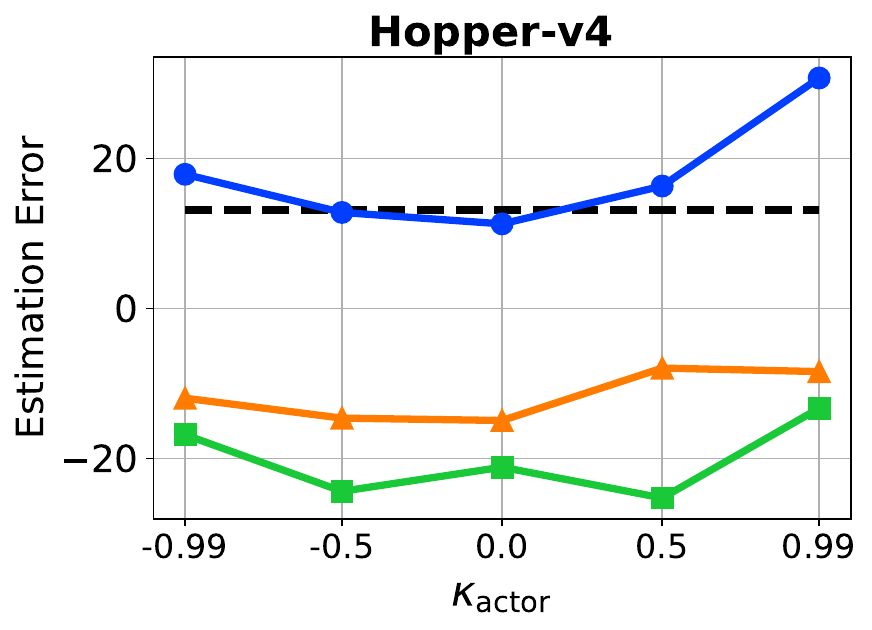}
\includegraphics[width=0.195\textwidth]{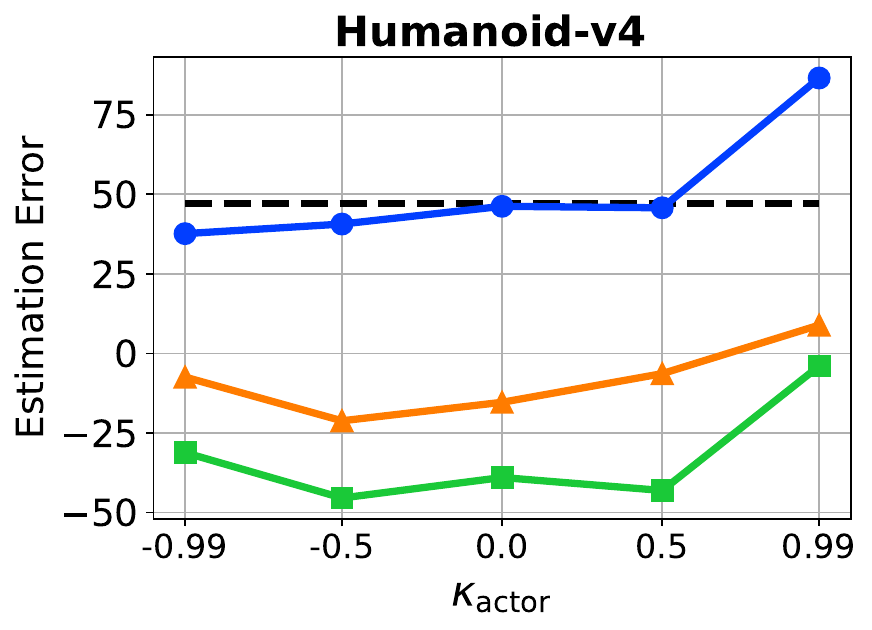}
\includegraphics[width=0.195\textwidth]{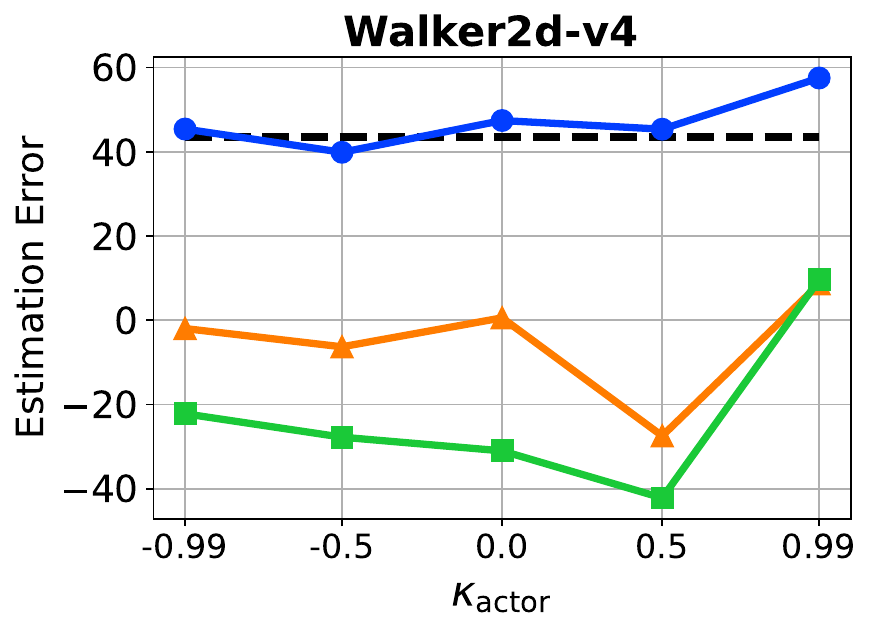}
\caption{The impact of the degree of actor optimism on the estimation error (true-estimation) of USAC algorithm when Laplace distributed critics are assumed. The results are averaged over three seeds, each with three evaluation episodes. The experiments span various $\kappa_{\text{critic}}$ and $\kappa_{\text{actor}}$ values, using SAC as the baseline.}
\label{fig:grid_estimation_error}
\end{figure*}
\begin{figure*}[th!]
\centering
\includegraphics[width=0.195\textwidth]{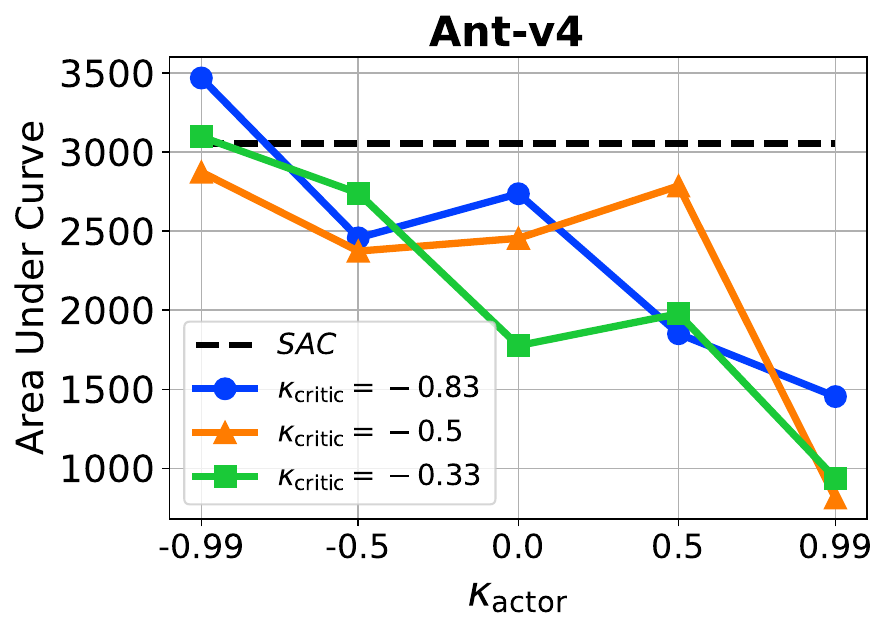}
\includegraphics[width=0.195\textwidth]{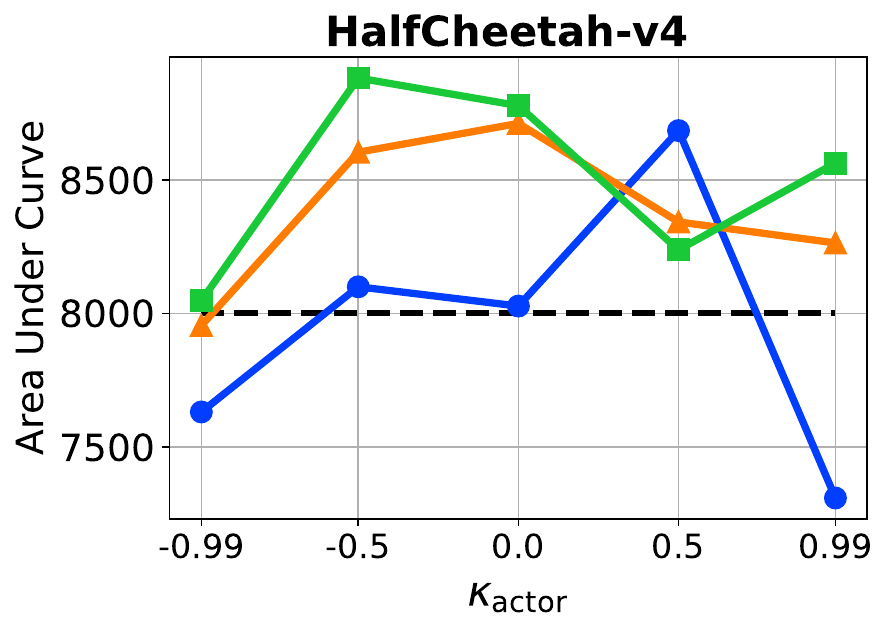}
\includegraphics[width=0.195\textwidth]{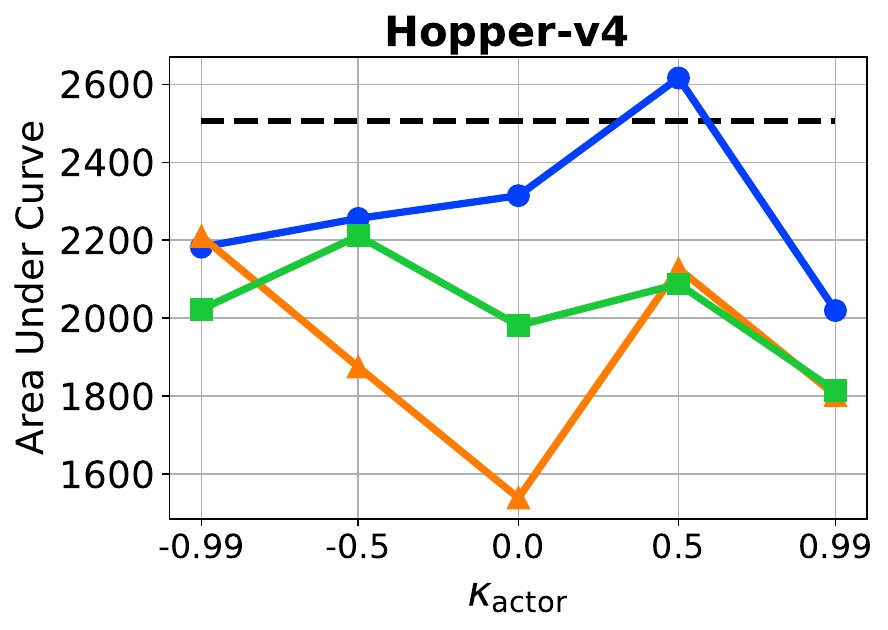}
\includegraphics[width=0.195\textwidth]{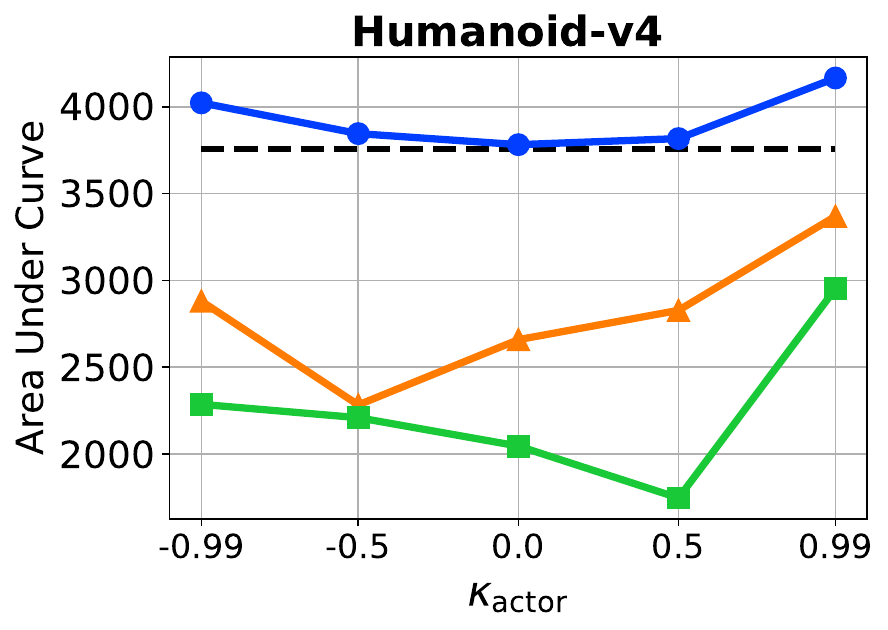}
\includegraphics[width=0.195\textwidth]{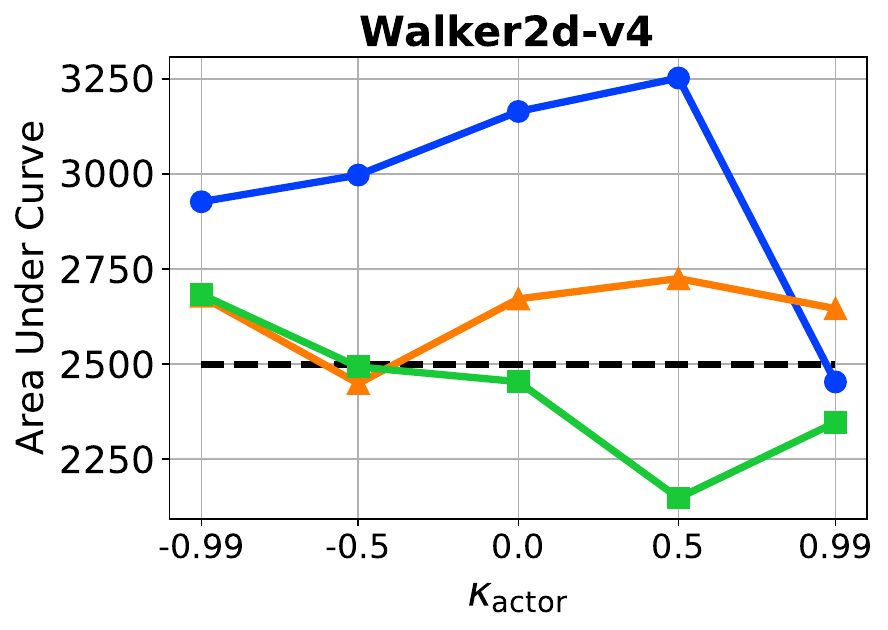}
\caption{The impact of the degree of actor optimism on the learning speed of USAC algorithm when Laplace distributed critics are assumed, measured by the area under learning curve. The results are averaged over three seeds, each with three evaluation episodes. The experiments span various $\kappa_{\text{critic}}$ and $\kappa_{\text{actor}}$ values, using SAC as the baseline.}
 \label{fig:grid_auc}
\hspace{0.5em}
\end{figure*}

\subsection{Question 1: Can USAC hyperparameters explain the optimism and pessimism dynamics?}
We investigate the performance profile of USAC with Laplace distributed critics  (\cref{sec:USAC-lap}) across different degrees of optimism and pessimism on five MuJoCo environments. We expect to replicate the findings of prior art \citep{fujimoto2018addressing,ciosek2019better,moskovitz2021tactical}, thereby verifying our claim that the trade-off problem can be reduced to tuning an interpretable hyperparameter.
Results are averaged over three different seeds, comparing against SAC as a baseline, as it is a special case of USAC with $\kappa_{\text{critic}}=\kappa_{\text{actor}}=-0.831559$. This allows us to systematically analyze how different configurations of utility parameters affect learning behavior and overall returns.
We examined critics ranging from pessimistic to semi-optimistic and actors from pessimistic to optimistic. Specifically, we used $\kappa_{\text{critic}} \in \{-0.831559, -0.5, -0.33\}$\footnote{Recall that $\kappa = -0.831559$ corresponds to $g(\kappa) = 1$, simplifying the utility function in \cref{eq:utility_function:laplace} to the minimum of the two critics. Therefore, by setting $\kappa_{\text{critic}} \approx -0.83$, we align with the pessimistic choices of SAC \citep{haarnoja2018soft}, TD3 \citep{fujimoto2018addressing}, and OAC \citep{ciosek2019better}.}, and ${\kappa_{\text{actor}} \in \{-0.99, -0.5, 0, 0.5, 0.99\}}$. The rationale for choosing more pessimistic (to semi-optimistic) $\kappa_{\text{critic}}$ values is to counteract the overestimation bias commonly seen in off-policy reinforcement learning. 
We present results for three metrics that provide complementary information: (i) {\bf final return} (\cref{fig:grid_rewards}) to assess the success of each configuration at solving the task at hand,
(ii) {\bf   estimation error} (\cref{fig:grid_estimation_error}) to assess the explanatory value of our framework for further investigation, (iii) {\bf area under the learning curve} (\cref{fig:grid_auc}) to assess the learning speed of each configuration. 

The estimation error is calculated by subtracting the estimated $Q$-value from the true (discounted) reward. Thus, positive values indicate underestimating, while negative values indicate overestimation, both of which can impact learning efficiency and stability.
Our results verify our hypothesis that the following key findings of the prior art can indeed be explained by a single hyperparameter value:
\begin{itemize}
    \item  \textbf{Overestimation} (i.e., negative estimation error) should always be avoided as it reduces performance, as visible from the strong correlation between high estimation error in \cref{fig:grid_estimation_error} and low final return in \cref{fig:grid_rewards}.
    \item \textbf{Underestimation} (i.e., positive estimation error) does not necessarily harm performance. This outcome reaffirms the central claims reported in the TD3 paper \citep{fujimoto2018addressing}, which yielded the wide adoption of the min-clipping method.
\end{itemize}
We also conclude from the results that a pessimistic approach should be adopted in critic training, whereas in actor training, either a pessimistic or optimistic approach can be suitable depending on the environment and the exploration-exploitation balance required. This outcome reaffirms the prior findings \citep{fujimoto2018addressing,ciosek2019better,moskovitz2021tactical}, supporting our hypothesis about the possibility of expressing the optimism and pessimism dynamics on a simple and interpretable axis.

\begin{table}[htb] 
\caption{The best pair of hyperparameters 
$(\kappa_{\text{critic}}, \kappa_{\text{actor}})$ for USAC with respect to final returns; see for example \cref{fig:grid_rewards}.} 
\label{tab:best pair}
\centering 
\begin{tabular}{llcc}
\\ \toprule
 & Environment & $\kappa_{\text{critic}}$ & $\kappa_{\text{actor}}$  \\
\midrule
 & \texttt{Ant-v4} & $-0.831559$ & $-0.99$\\
 & \texttt{HalfCheetah-v4} & $-0.33$ & $-0.50$ \\
 & \texttt{Hopper-v4} & $-0.831559$ & $0.50$  \\
 & \texttt{Humanoid-v4} & $-0.831559$ & $-0.831559$ \\ 
 & \texttt{Walker2d-v4} & $-0.831559$ & $0.0$  \\
\bottomrule
\end{tabular}
\end{table}

\subsection{Question 2: Can the hyperpararameters be effortlessly tuned to improve performance?}

From \cref{fig:grid_rewards,fig:grid_estimation_error,fig:grid_auc}, we can extract the best combinations of $(\kappa_{\text{critic}}, \kappa_{\text{actor}})$ per environment with respect to final return and treat them as the outcome of a search on a coarse grid of hyperparameters. These best values are listed in \cref{tab:best pair}. The difference between the best values across the environments highlights the environment-specific nature of optimism and pessimism and also demonstrates the success of our framework at detecting these differences.

\begin{table*}[t!]
\caption{Final return on MuJoCo environments trained with 1M time steps, averaged over 5 seeds. The best algorithms are highlighted in \textbf{bold}. $\pm$ corresponds to the standard deviation across repetitions. $(\kappa_{\text{critic}},\kappa_{\text{actor}})$ for USAC are listed in \cref{tab:best pair}.} 
\label{tab:results}
\centering
\begin{tabular}{lcccccc}
\\ \toprule
Environment & USAC (\textbf{ours}) & SAC & TD3 & OAC & TOP
\\ \midrule
\texttt{Ant-v4} & $\boldsymbol{ 5158 \pm 424}$& $4756 \pm 773$ & $4091 \pm 130$ & $4177 \pm 802$ &$ 4334\pm 1276$\\
\texttt{HalfCheetah-v4}  & $\boldsymbol{ 11736 \pm 297}$ & $10763 \pm 891$ & $10570 \pm 801$ & $8684 \pm 1671$ & $7311 \pm 3074$\\
\texttt{Hopper-v4} & $\boldsymbol{3442 \pm 126}$ & $ 3185 \pm 380$ &  $1986 \pm 1114$ & $ 3293 \pm 112$ & $ 3369 \pm 154$ \\
\texttt{Humanoid-v4} &  $\boldsymbol{ 5602 \pm 210}$ & $ 5503 \pm 198$ & $5149 \pm 221$ & $ 5390 \pm 121$& $5332 \pm 445$ \\ 
\texttt{Walker2d-v4} &  $\boldsymbol{ 4530 \pm 708}$ & $3757 \pm 1069$ & $4369 \pm 594$ & $3467 \pm 1193$& $ 4317 \pm 631$\\
\bottomrule
\end{tabular}
\end{table*}

\begin{figure*}[h!]
\centering
\includegraphics[width=0.33\textwidth]{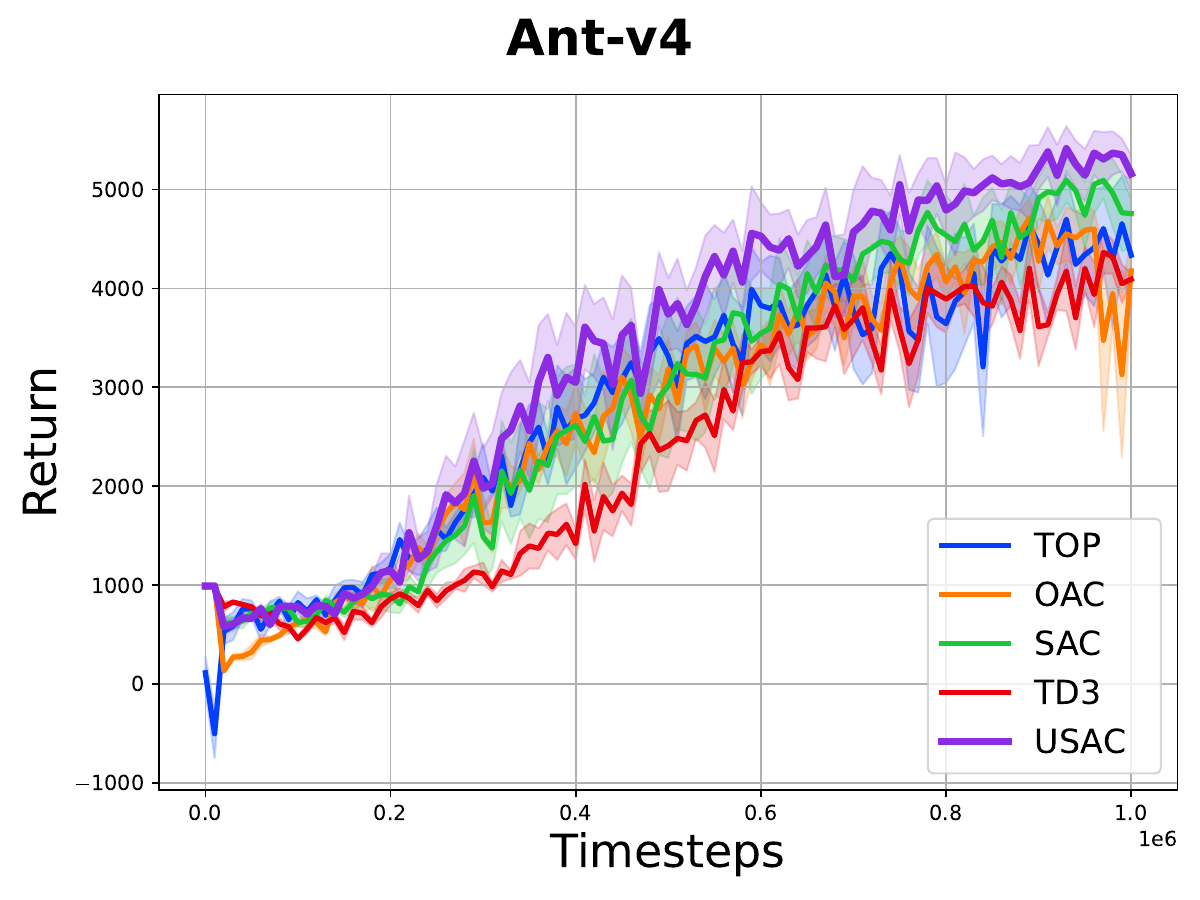}
\includegraphics[width=0.33\textwidth]{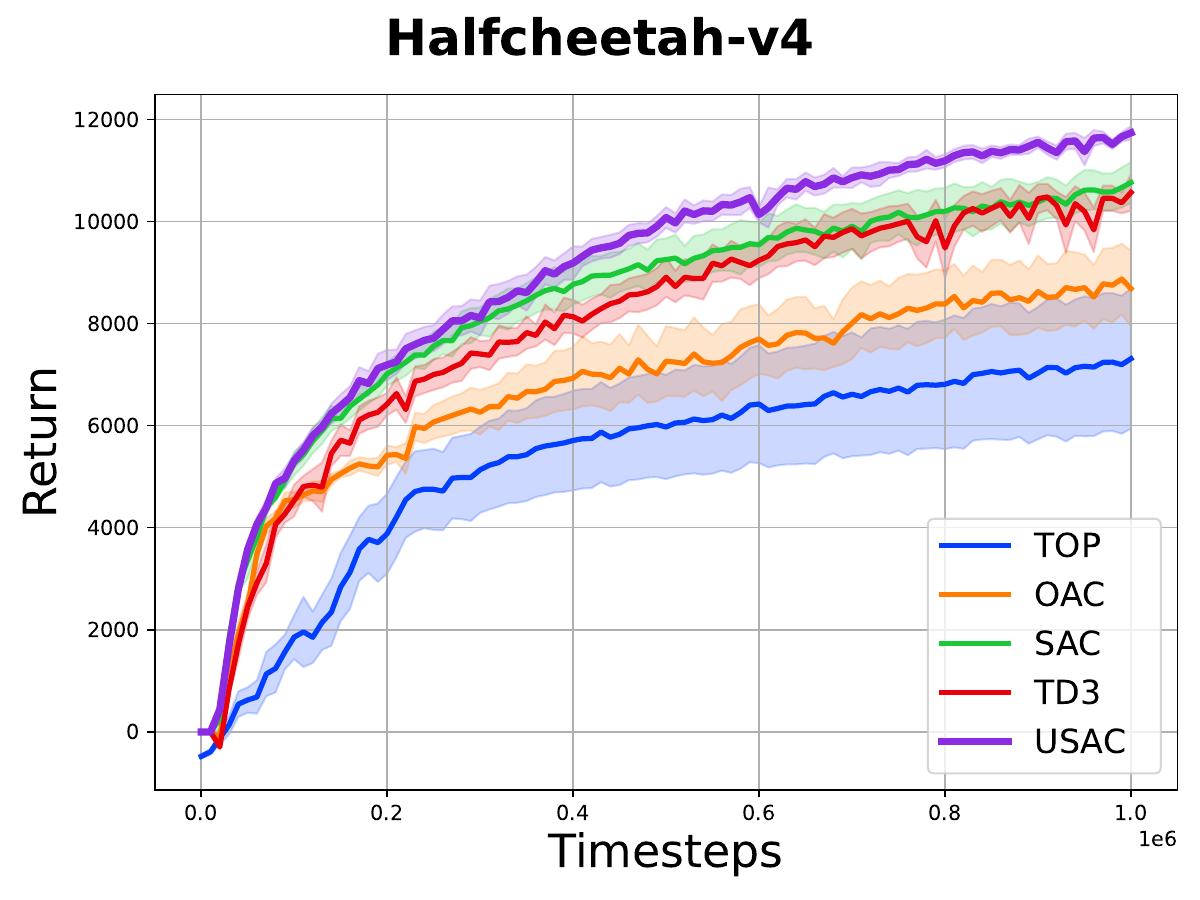}
\includegraphics[width=0.33\textwidth]{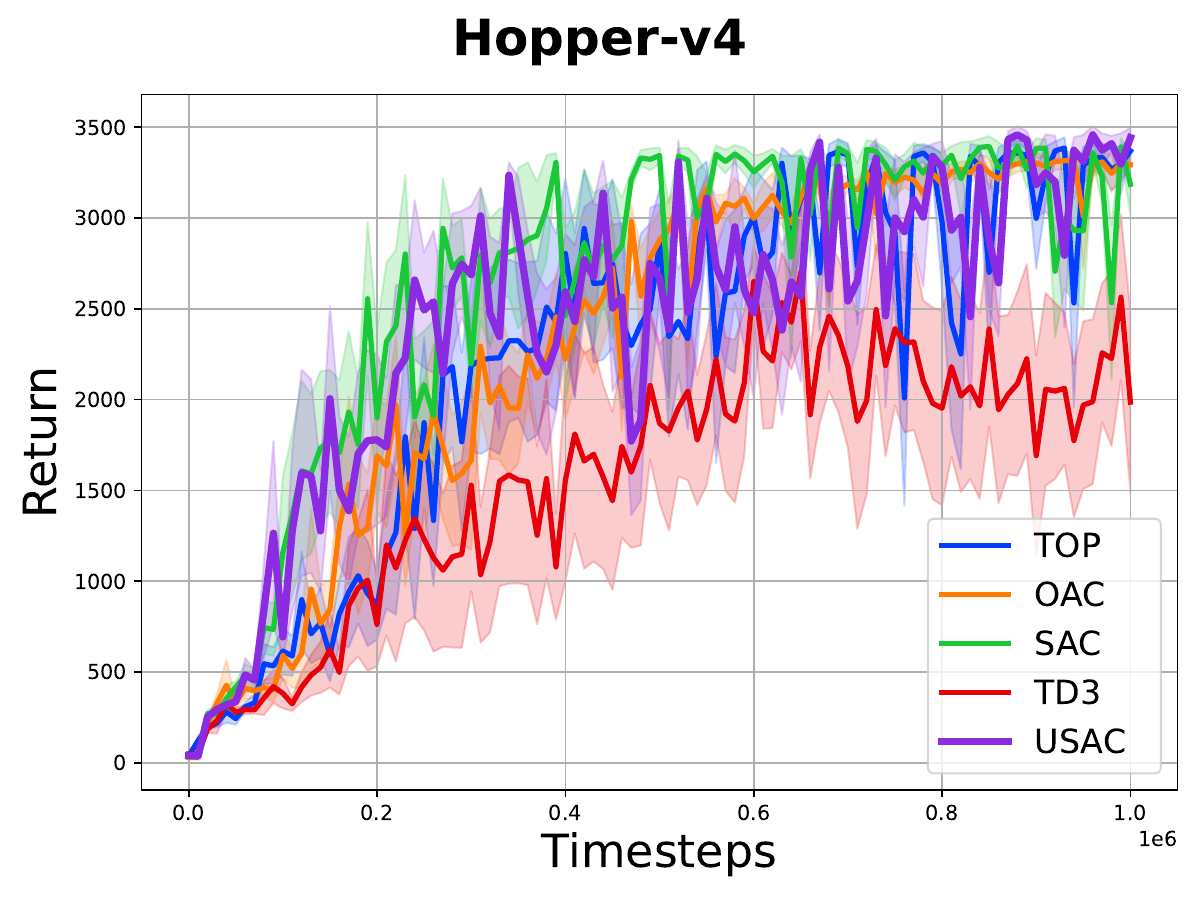}
\includegraphics[width=0.33\textwidth]{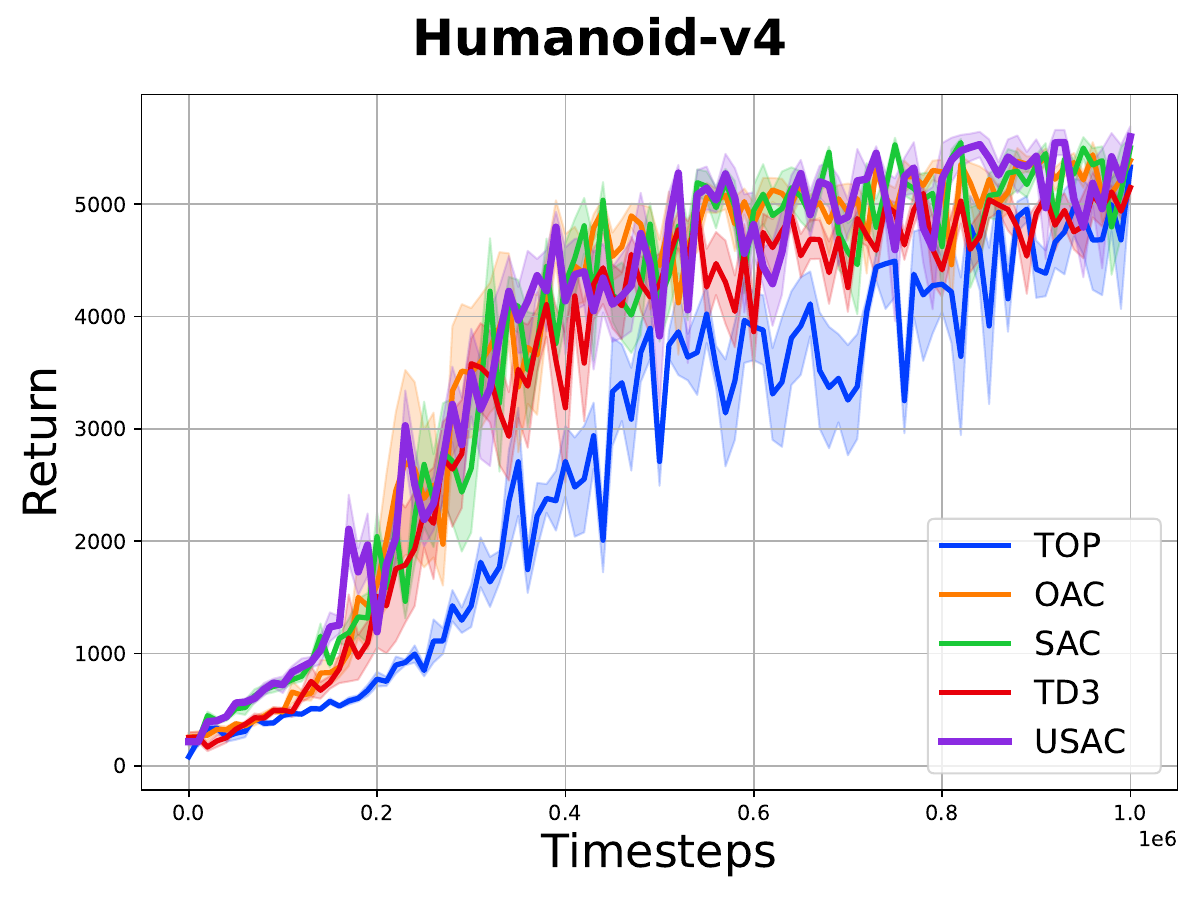}
\includegraphics[width=0.33\textwidth]{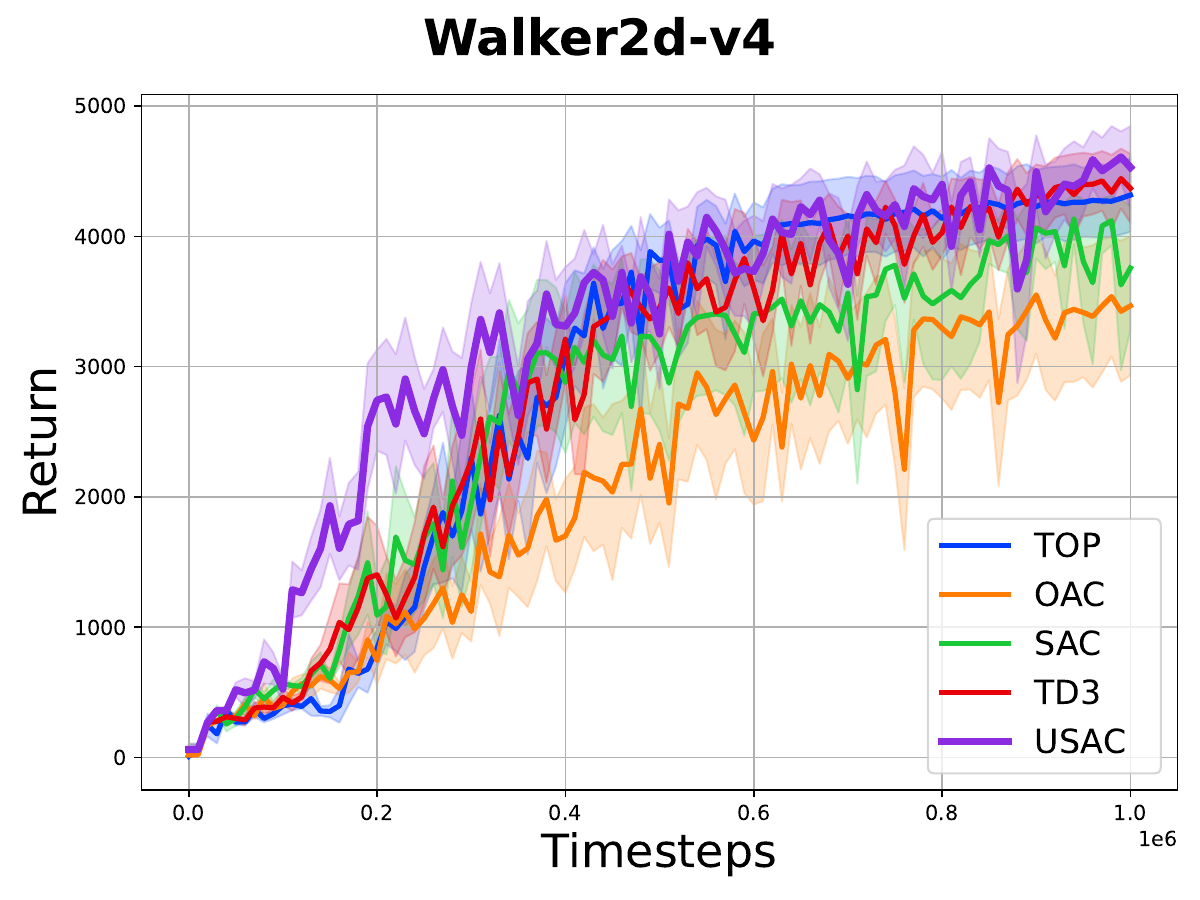}
\caption{Learning curves of USAC and baselines. Solid curves depict the average return across evaluation episodes, while the shaded areas represent the standard deviation.}
 \label{fig:final-rewards}
 \hspace{0.5em}
\end{figure*}

We give a performance comparison of the grid-optimized USAC against the baseline models in \cref{tab:results}.  The superior performance of USAC supports our hypothesis that the hyperparameters introduced by our framework also have an instrumental value, as they can be used to improve actor-critic algorithms. It is also remarkable that USAC incurs a smaller standard deviation than the baseline models across repetitions in the far majority of the cases, indicating an improved model stability during learning.  The learning curve of the grid-optimized USAC algorithm is given in \cref{fig:final-rewards} together with the baseline algorithm. The results suggest a consistent performance improvement that spans the whole learning regime.

In applications where online learning with minimum regret is more important than maximizing the final reward, such as interactive learning applications, a similar grid-search may also be performed with respect to the area under curve. While we omit a detailed reporting of such an analysis for brevity, one can see from \cref{fig:grid_auc} that USAC will be similarly competitive in that scenario.

\section{Discussion and limitations} \label{sec:discussion}

We introduced a novel algorithmic framework to trade pessimistic policy evaluation and pessimistic/optimistic policy search in off-policy deep actor-critic pipelines. The key idea is to balance optimism and pessimism through incorporating a utility function that captures the uncertainty in the critic due to limited access to the environment into a single and interpretable hyperparameter. This USAC framework is flexible and can be adapted to different distributions, with our experiments primarily focusing on the Laplace distribution. 

Our USAC framework allows independent control over the optimism and pessimism levels for the actor and critic. This decoupling enables a stable critic, necessary for accurate value function estimation, and an explorative actor, crucial for effective policy learning. Furthermore, USAC allows tuning the degree of optimism/pessimism via a single hyperparameter that interpretably encodes full pessimism as -1 and full optimism as +1 and all intermediate possibilities as values in between. Our experiments revealed that being pessimistic in critic training while being either optimistic or pessimistic in actor training can be beneficial. However, it is also essential to consider the specific environment. We also found that a search on a coarse grid of our new hyperparameters brings consistent performance gains across environments with diverse characteristics.

The primary limitation of USAC is that it does not prescribe any insights for on-the-fly tuning of its hyperparameters. An exciting direction for future research is to develop adaptive schemes where these parameters adjust dynamically throughout the learning process to better respond to the changing dynamics of the environment. There are several potential approaches to integrate automatic tuning. For example, (i) adapting the automatic tuning scheme used for $\alpha$ in SAC \citep{haarnoja2018soft,haarnoja2018softa}, (ii) employing bandit-based selection methods \citep{moskovitz2021tactical}, or (iii) exploring continuous learning approaches using gradient descent \citep{bharadwaj2023continuous}.

Another limitation is our evaluation scope. We focus on the five most used MuJoCo environments \citep{todorov2012mujoco}, where the studied phenomenon and the impact of our solution are clearly visible, ensuring comparability with the key baselines we consider. Broader validation across additional domains would further demonstrate generality, which we plan to explore in future work.  

Our findings could also be improved by incorporating higher moments of critic distributions (mean, variance, skewness, kurtosis) to better capture uncertainty in value estimates. Lastly, extending from critic pairs to ensembles may offer faster training at the cost of increased computational power.

\newpage

\section*{Acknowledgements}
This work was funded by the Novo Nordisk Foundation (NNF21OC0070621) and 
the Carlsberg Foundation (CF21-0250).
\bibliography{references}

\end{document}